\documentclass{article}
\usepackage{verbatim}
\usepackage{amsmath,amssymb}
\usepackage{arxiv}
\usepackage{diagbox}
\usepackage{subcaption}
\usepackage{graphicx}
\usepackage{amsthm}
\usepackage[utf8]{inputenc} 
\usepackage[T1]{fontenc}    
\usepackage{hyperref}       
\usepackage{url}            
\usepackage{booktabs}       
\usepackage{amsfonts}       
\usepackage{nicefrac}       
\usepackage{microtype} 
\usepackage{shuffle}
\usepackage{lipsum}
\usepackage{xcolor}
\newtheorem{theorem}{Theorem}
\newtheorem{corollary}{Corollary}
\newtheorem{definition}{Definition}

\newtheorem{rmk}{Remark}

\usepackage{mathtools} 

\title{Signature features with the visibility transformation}

\author{
  Yue Wu \\
  Mathematical Institute\\
  University of Oxford\\
  Oxford,  OX2 6GG, UK\\ 
  Alan Turning Institute\\
  London, UK\\
  \texttt{yue.wu@maths.ox.ac.uk} \\
       \And 
   Hao Ni \\
 Department of Mathematics\\
   University College of London\\
 Alan Turning Institute\\
  London, UK\\
  \texttt{ h.ni@ucl.ac.uk}  \\
     \And 
        Terence J. Lyons \\
 Mathematical Institute\\
  University of Oxford\\
  Oxford, OX2 6GG, UK \\
    Alan Turning Institute\\
  London, UK\\
  \texttt{lyons@maths.ox.ac.uk}  \\
        \And 
   Robin L. Hudson \\
 Department of Mathematical Sciences\\
  Loughborough University\\
  Loughborough, LE11 3TU, UK \\
  \texttt{r.hudson@lboro.ac.uk}  \\
}

\begin{document}
\maketitle

\begin{abstract}
In this paper we put the visibility transformation on a clear theoretical footing and show that this transform is able to embed the effect of the absolute position of the data stream into signature features in a unified and efficient way. The generated feature set is particularly useful in pattern recognition tasks, for its simplifying role in allowing the signature feature set to accommodate nonlinear functions of absolute and relative values.
\end{abstract}


 \section{Introduction}\label{sec:intro}

Feature extraction is the key to effective model construction in the context of machine learning. Real-world complex data always comes with lots of inherent noise and variation, thus a good and unified choice of feature is needed to provide informative resources, and to facilitate subsequent learning. The focus of this paper, namely, the visibility transformation, designed to retain information about value of the position within the corresponding signature feature, is able to capture comprehensive information of the data stream hidden in both increments and positions in one shot.

It is well-known that the signature of streamed data, which consists of an infinite sequence of coordinate iterated integrals, makes use only of the increments of the path generated from the streamed data rather than the absolute values. With its nature to capture the total ordering of the streamed data and to summarise the data over segments,  signature-based machine learning models have proved efficient
 in several fields of application, from automated recognition of Chinese handwriting \cite{gaham2013sparse, xie2017learning} to diagnosis of mental health problems \cite{moore2019using, wang2019speech,wang2020speech}. However, for some scenarios, handwriting recognition and human action recognition for example, the position information is  informative for characterising the temporal dynamics as well. 
It is therefore quite demanding to introduce some transform that can preserve the effects of both increments and positions simultaneously. The first attempt was to use the so-called the \emph{invisibility-reset transformation} in the numerical experiments in \cite{yang2017skeleton} for skeleton-based human action recognition tasks. For a discrete data stream, this discrete transformation is able to incorporate its initial position value into signature features but the authors 
were not aware of the nonlinear effect of the tail position being captured at the same time and failed to disclose the true (and continuous) form of this discrete transformation via piecewise linear interpolation. In this paper, we generalise the idea of invisibility-reset transformation to the unified framework of visibility transformation for incorporating different position points of the path, and provide theoretical justifications. The visibility transformation, through translating the path to a new path, offers a new way of preparing datasets and does not need to change the pipeline (see the workflow Figure \ref{fig:workflow}). Its ability in capturing nonlinear effects on path segments and path positions simultaneously (see Theorem \ref{thm:general2} and Theorem \ref{thm:general1}), sheds a light on its better performance compared to the performance attained by using the signature alone in certain circumstances.   The availability of the well-established Python packages for calculating signature features from data streams allows easy implementation of extraction of features using the visibility transformation. Owing to the fundamental nature of the framework, we foresee a multifaceted impact in data-driven pattern recognition applications. 

The paper is organised as follows. In Section \ref{sec:pre} the relevant foundations concerning the signature are reviewed briefly. In Section \ref{sec:visibility} we formulate the visibility transformation map for bounded variation paths using the concatenation operator, and discuss its ability to capture the effects of the positions as well as increments. Its discrete version for the streamed data is introduced in Section \ref{sec:visibility2}, and then assessed in different pattern recognition applications in Section \ref{sec:app}. We conclude our paper in Section \ref{sec:con}. All the proofs are postponed to the Appendix.

\begin{figure*}[!t]
\centering
\includegraphics[trim=2.3cm 22cm 2.3cm 4cm, clip,width=6in]{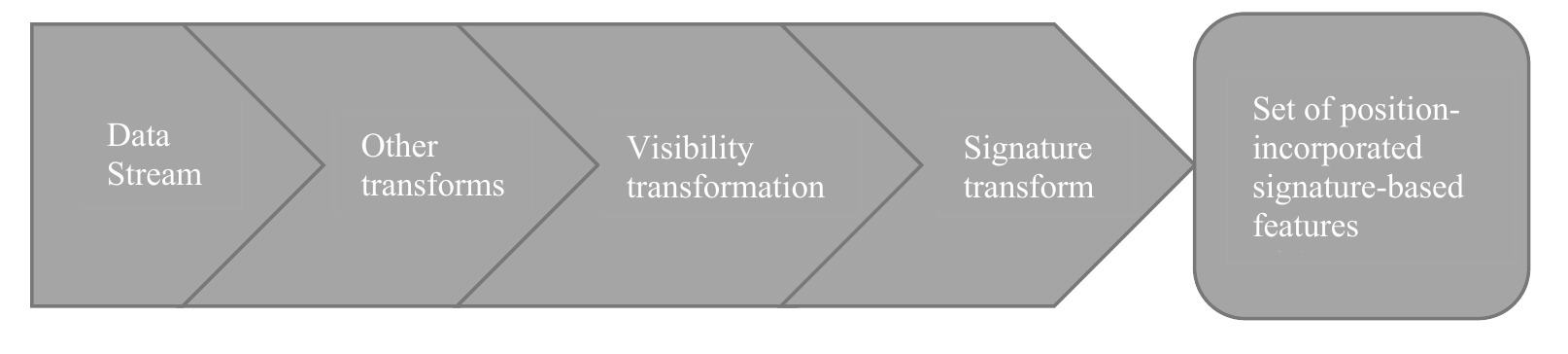}
\caption{The workflow of the pipeline for feature extraction using visibility transformation: the input is the original data stream on the left end, and one may utilise different transforms on the data stream for data cleaning and scaling; this is followed by applying the visibility transformation on the cleaned data; finally one will use package to translate the transformed data to signature.}
\label{fig:workflow}
\end{figure*}

\section{Preliminaries in signatures} \label{sec:pre}
We consider $\mathbb{R}^d$-valued time-dependent, piecewise-differentiable paths of finite length. Such a path $X$ mapping from $[a,b]$ to $\mathbb{R}^d$ is denoted as $X:[a,b] \to \mathbb{R}^d$. Denote by $\mathbf{I}(X)$ the initial position $X(a)$ of path $X$ and $\mathbf{T}(X)$ the tail position $X(b)$ of path $X$. For short we will use $X_t$ for $X(t), t\in [a,b]$. Each coordinate path of $X$ is a real-valued path and denoted as $X^i, i\in [d]$ with $[d]:=\{1,\ldots,d\}$. Now for a fixed ordered multi-index collection $(i_1,\ldots,i_{k})$, with $k\in \mathbb{N}$ and $i_j\in [d]$ for $j \in [k]$, define the coordinate iterated integral by
{\small  \begin{equation}\label{eqn:iterated_integral}
S(X)_{a,t}^{i_1,\ldots,i_{k}}:=\int_{a<t_{k}<t}\cdots \int_{a<t_{1}<t_2} \text{d}X_{t_1}^{i_1}\ldots \text{d}X_{t_k}^{i_k},
\end{equation}}
where the subscript $a,t$ denotes the lower and upper limits of the integral. It is easy to verify the recursive relation
{\small  \begin{equation}
S(X)_{a,t}^{i_1,\ldots,i_{k}}=\int_{a<t_{k}<t} S(X)_{a,t_{k}}^{i_1,\ldots,i_{k-1}}\text{d}X_{t_{k}}^{i_{k}}.
\end{equation}}
\begin{definition}\label{def:signature}
The \emph{signature} of a path $X:[a,b]\to \mathbb{R}^d$, denoted by $S(X)_{a,b}$, is the infinite collection of all iterated integrals of $X$. That is, 
{\small  \begin{equation}\label{eqn:signature}
    S(X)_{a,b}:=(1,S(X)_{a,b}^1, \ldots,S(X)_{a,b}^d, S(X)_{a,b}^{1,1}, S(X)_{a,b}^{1,2},\ldots),
\end{equation}}
where, the $0$th term is 1 by convention, and the superscripts of the terms after the $0$th term run along the set of all multi-index
$\{(i_1,\ldots,i_k)|k\geq 1, i_1,\ldots,i_k \in [d]\}$. The finite collection of all terms $S(X)_{a,b}^{i_1,\ldots,i_k}$ with the multi-index of fixed length $k$ is termed as the \emph{kth level of the signature}. The truncated signature up to the $p$th level is denoted by $\lfloor S(X)_{a,b} \rfloor_p$.
\end{definition}
It is not hard to deduce that the length of the signature up to level $p$ of a $d$-dimensional path is $\frac{d(d^p-1)}{d-1}$. In practice, truncating the signature at given level transforms input data of different lengths into one-dimensional feature vectors of the same length.

The definition also reveals that the signature only captures the effect of pattern change and not ones depending on the absolute position \cite{lyons2014signature}. This is further supported by the following calculation from the definition: 
{\small  \begin{equation}
S(X)_{a,b}^{\overbrace{l,\ldots,l}^{k}}=\frac{1}{k!}(X^l_b-X^l_a)^k,\text{\ for\  }k\in \mathbb{N}\ \text{and\ }l\in [d].
\end{equation}}
Note these terms of the signature are completely described by the increments of the coordinates on the right hand side.

Alternatively, the signature of a path $S(X)$ can be viewed as a non-commutative polynomial on the path space. This leads to the following representation of $S(X)$ 
 \begin{equation}\label{eqn:power_rep}
    S(X)_{a,b}= \sum_{k=0}^{\infty} \sum_{i_1,\ldots,i_k\in [d]}S^{i_1,\ldots,i_k}_{a,b} e_{i_1}\cdots e_{i_k},
\end{equation}
where $S^{i_1,\ldots,i_k}_{a,b}$ are coefficients of $S(X)_{a,b}$ and $e_{i_1}\ldots e_{i_k}$ are \emph{monomials} in the \emph{tensor algebra} of $\mathbb{R}^d$.
This gives rise to the multiplicative property of the signature called Chen's identity (c.f. \cite{chen1958chen, lyons2002control}), which is crucial for proving one of the main results Theorem \ref{thm:general2}.
Before proceeding to Chen's identity, we introduce two more important concepts for paths.
\begin{definition}\label{def:con1}
Given two continuous paths $X:[a,b]\to \mathbb{R}^d$ and $Y:[c,d]\to \mathbb{R}^d$ with $\mathbf{I}(X)=\mathbf{T}(Y)$. The \emph{concatenation product} $X*Y:[a,b+d-c]\to \mathbb{R}^d$ is the continuous path and defined by
{\small  \begin{equation}
 X*Y(t):=
\begin{dcases}
X(t), &t\in [a,b] \\
Y(t+c-b), & t\in [b,b+d-c].
\end{dcases}   
\end{equation}}
Also the \emph{reversal operation} $\stackrel{\leftarrow}{X}:[a,b]\to \mathbb{R}^d$ is defined by
{\small \begin{equation}
\stackrel{\leftarrow}{X}(t):=X(a+b-t)\ \ \text{for}\ t\in [a,b]. 
\end{equation}}
\end{definition}
We can now present the classical Chen's identity as follows.  
\begin{theorem}[Chen's identity] \label{thm:chen} Given two continuous paths $X:[a,b]\to \mathbb{R}^d$ and $Y:[c,d]\to \mathbb{R}^d$ such that $\mathbf{I}(X)=\mathbf{T}(Y)$. 
Then
{\small \begin{equation}\label{eqn:chen}
  S(X*Y)_{a,b+d-c}=S(X)_{a,b}\otimes S(Y)_{c,d}.  
\end{equation}}
\end{theorem}
Thus the multiplicative property of the signature of a path is preserved under concatenation and the signature of the entire path can be captured by calculating the signatures of its pieces. By now the signature map S is revealed as a homomorphism of the monoid of paths (or path segments) with concatenation into the tensor algebra.

Theorem \ref{thm:chen} also leads to the fact that a bounded variation path which is completely cancelled out by itself has null effect on its increments, i.e., $S(X*\stackrel{\leftarrow}{X})=(1,0,0,\ldots).$
\subsection{Signatures as features}
The rough path theory shows that the solution of a controlled system driven by path $X$ is uniquely determined by its signature and the initial condition \cite{lyons2007differential}.
For a path of finite length, the corresponding signature is the fundamental representation that captures its effect on any nonlinear system and ensures that the effects of paths or streams can be locally approximated by linear combinations of signature elements.
Therefore, the coordinate iterated integrals, or the signature in total, are a natural feature set for capturing the aspects of the data that predict the effects of the path on a controlled system. 
The signature can remove the infinite dimensional redundant information caused by time reparameterisation while retaining the information on the order of events \cite{lyons2007differential}. Moreover, the signature of a path, as a feature set, has the advantages of being able to handle time of variable length, unequal spacing and missing data in a unified way \cite{liao2019learning}  \cite{wu2020mentalhealth}.

On the other hand, the signature map $S$ is not one-to-one, but its kernel is well understood. Two distinct paths can have exactly the same signature. For example, they are the same under time reparametersation. To characterise those paths with the same signature, we use a \emph{geometric} relation $\sim$ called \emph{tree-like equivalence} on paths of finite length \cite{hambly2010uniqueness}.

In order to visualise tree-like equivalence, in Figure \ref{fig:threelike} we exhibit two 2 dimensional paths which are tree-like equivalent as an example. It can be seen that both of the curves have the same shape expect that the right one has some "new part" that is completely self-cancelling. 

\begin{figure}[h!]
  \centering
\begin{subfigure}[b]{0.48\linewidth}
\includegraphics[trim=3cm 16.5cm 4.5cm 3.5cm, clip,width=1.0\textwidth]{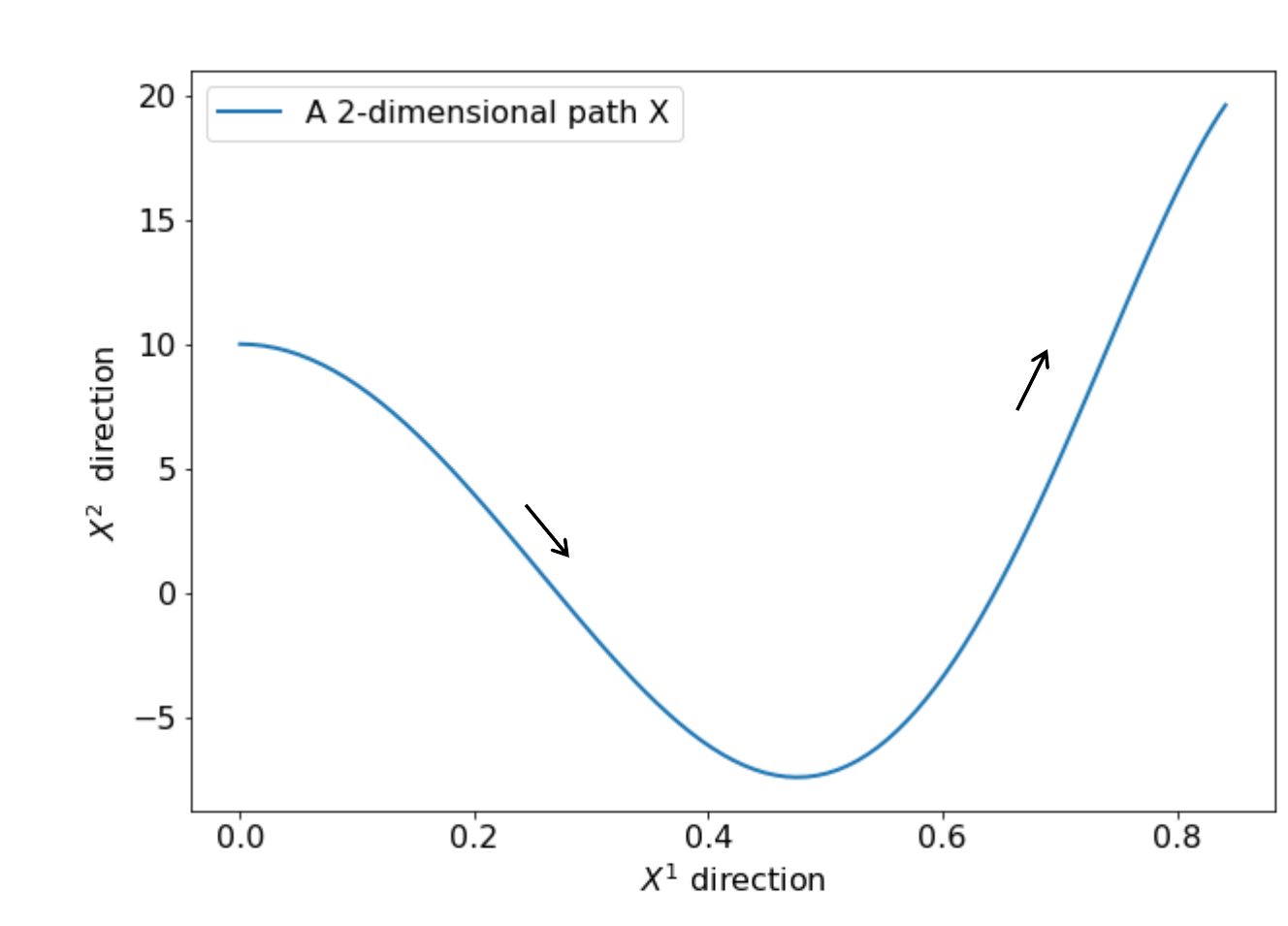}
\caption{A 2-dimensional path $X$.}
\end{subfigure}
\begin{subfigure}[b]{0.48\linewidth}
\includegraphics[trim=3cm 16.5cm 4.5cm 3cm, clip,width=1.0\textwidth]{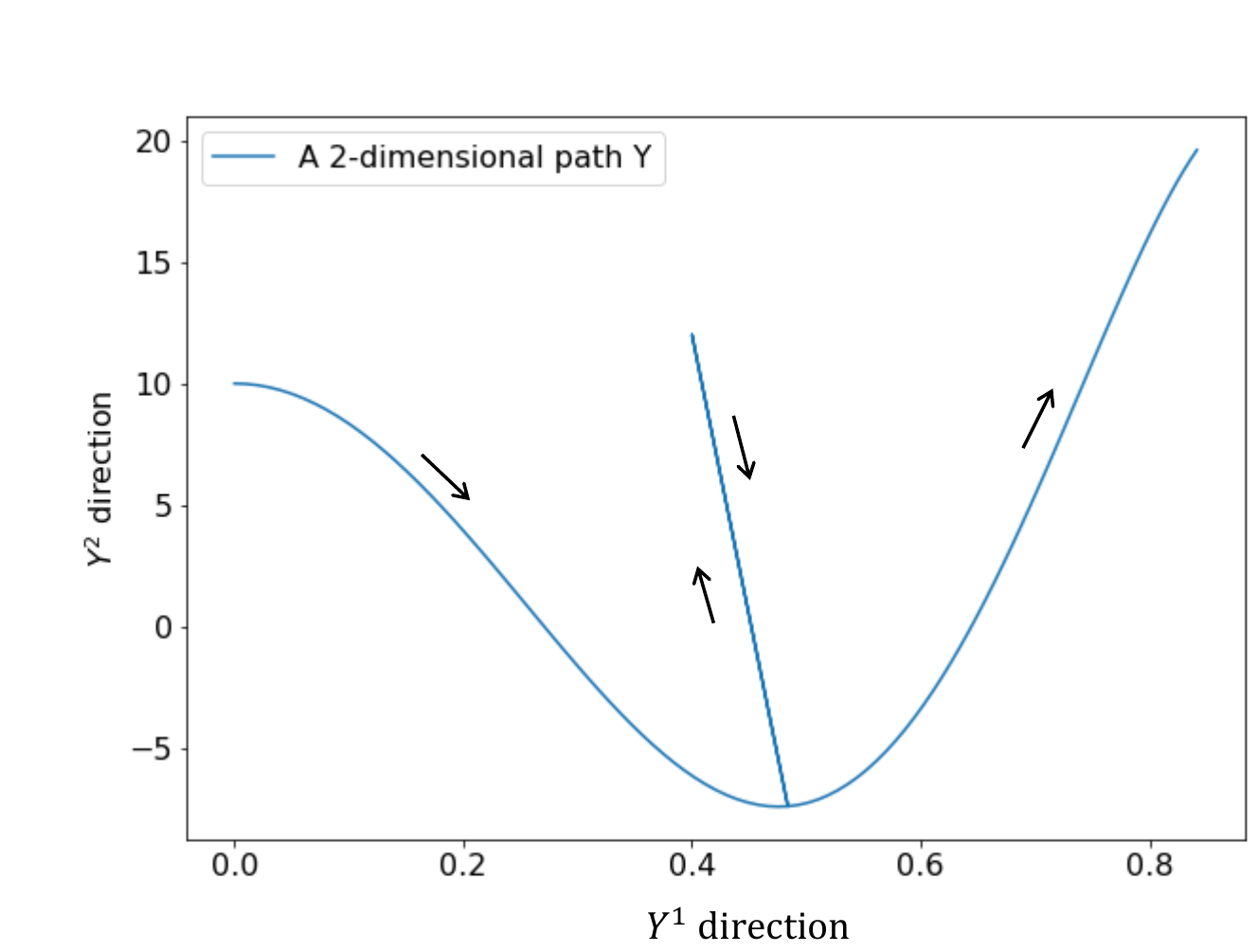}
\caption{A 2-dimensional path that is tree-like equivalent with $X$.}
\end{subfigure}
\caption{An illustration for tree-like equivalence: the left plot is a 2-dimensional curve $X$ on time range $[0,1]$, where the $X^1(t)=\sin(t)$, and $X^2(t)=t^2+\cos(10t)$; the right plot is a 2-dimensional curve $Y$ on the same time range, with $Y^1(t)=\sin(2t)$, $Y^2(t)=4t^2+\cos(20t)$ for $t\in [0,0.25]$, , $Y^1(t)=k_1(t-0.25)+\sin(0.5)$, $Y^2(t)=k_2(t-0.25)+0.25+\cos(5)$ for $t\in [0.25,0.5]$, $Y^1(t)=-k_1(t-0.5)+0.4$, $Y^2(t)=-k_2(t-0.5)+12$ for $t\in [0.5,0.75]$, $Y^1(t)=\sin(2t-1)$, $Y^1(t)=(2t-1)^2+\cos(20t-10)$ for $t\in [0.75,1]$, where $k_1=1.6-4\sin(0.5)$, and $k_2=47-4\cos(5)$.}
\label{fig:threelike}
\end{figure} 
Hambly-Lyons \cite{hambly2010uniqueness} shows that $X\sim Y$ is an equivalence relation, and the equivalence classes form a group under concatenation.  Hambly-Lyons also shows the tree-like property can be captured by checking the corresponding signature:
\begin{theorem}\cite{hambly2010uniqueness}\label{thm:sig_tree}
Given two $\mathbb{R}^d$-valued continuous parameterised paths $X$ and $Y$ with finite length such that $\mathbf{I}(X)=\mathbf{I}(Y)$ and $\mathbf{T}(Y)=\mathbf{T}(Y)$ . It holds that $X\sim Y$ if and only if $S(X*\stackrel{\leftarrow}{Y})=(1,0,0,\ldots)$. 
\end{theorem}
Theorem \ref{thm:sig_tree} ensures that the tree-like equivalence can be uniquely characterised by its signature. Put in other words, the increments of paths under tree-like equivalence have the same effects. It can be shown that a time-augmented path can be uniquely determined by and recovered from its signature \cite{levin2013past}. So there are no other path which has the same increment effects as this time-augmented path.

\section{The visibility transformation}\label{sec:visibility}
This section is devoted to the introduction of the visibility transformation. Section \ref{sec:visibility1} formulates the visibility transformation for a continuous path and Section \ref{sec:visibility2} considers the discrete version for practical use.
\subsection{The continuous path of finite length}\label{sec:visibility1}
To assist the understanding towards the visibility transformation, we first define the \emph{visibility plane} in $\mathbb{R}^{d+1}$ as
$\{[z_1,\cdots,z_d,1],z_j\in \mathbb{R}\ \ \text{for}\  j \in [d]\}$, and the \emph{invisibility plane} as 
$\{[z_1,\cdots,z_d,0],z_j\in \mathbb{R}\ \ \text{for}\  j \in [d]\}$. Without loss of generality, the time range is always assumed to be $[0,1]$ within this subsection. 
\begin{definition}\label{def:vt}
Given a continuous path $X:[0,1]\to \mathbb{R}^d$. The \emph{initial-position-incorporated visibility transformation} ($\mathbf{I}$-visibility transformation) $\gamma_{\mathbf{I}}$ maps the path $X$ to a $\mathbb{R}^{d+1}$ valued path starting at the origin,  where the path is determined by the continuous function $f_{\mathbf{I}(X)}*L_X$ with
{\small  \begin{equation}
f_{\mathbf{I}(X)}(t)=f_{X_0}(t):=[\tau(t)X^1_0,\cdots,\tau(t)X^d_0,\iota(t)]
\end{equation}}
and
{\small \begin{equation}
L_X(t):=[X^1_{t},\cdots,X^d_{t},1],
\end{equation}}
for $t\in [0,1]$, where $\tau(t):=\min(2t,1)$ and $\iota(t):=\max(2t-1,0)$. Similarly the \emph{tail-position-incorporated visibility transformation} ($\mathbf{T}$-visibility transformation) $\gamma_{\mathbf{T}}$ maps the path $X$ into a $\mathbb{R}^{d+1}$ valued path starting at the origin, where the path is determined by the continuous function $L_X*\stackrel{\leftarrow}{f_{\mathbf{T}(X)}}$, where  for $t\in [0,1]$,
{\small \begin{equation}
f_{\mathbf{T}(X)}=f_{X_1}(t):=[\tau(1-t)X_1^1,\ldots,\tau(1-t)X_1^d,\iota(1-t)].
 \end{equation}}
\end{definition}
In the definition of the \emph{$\mathbf{I}$-visibility transformation}, we construct two $\mathbb{R}^{d+1}$ segments from path $X$, and join them together using concatenation. The new path starts from the origin on the invisibility plane, and moves  to the initial position of path $X$ on the invisibility plane, i.e., $[X_0^1,\ldots,X_0^d, 0]$; then the path is made visible by being lifted onto the visibility plane, i.e., $[X_t^1,\ldots,X_t^d,1]$ for $t\in [0,1]$. By contrast, for the \emph{$\mathbf{T}$-visibility transformation}, the path is visible first and then invisible. Figure \ref{fig:loop} shows how a 2-dimensional path can be extended to a 3-dimensional path with invisibility and visibility information.
For simplicity, $f_{X_0}$, $L_X$ and $\stackrel{\leftarrow}{f_{X_1}}$ are defined on $[0,1]$. Indeed, the speed moving along the path does not affect its shape. 

\begin{figure}[h!]
  \centering
\begin{subfigure}[b]{0.48\linewidth}
\includegraphics[trim=3cm 14cm 4.5cm 3.5cm, clip,width=2.5in]{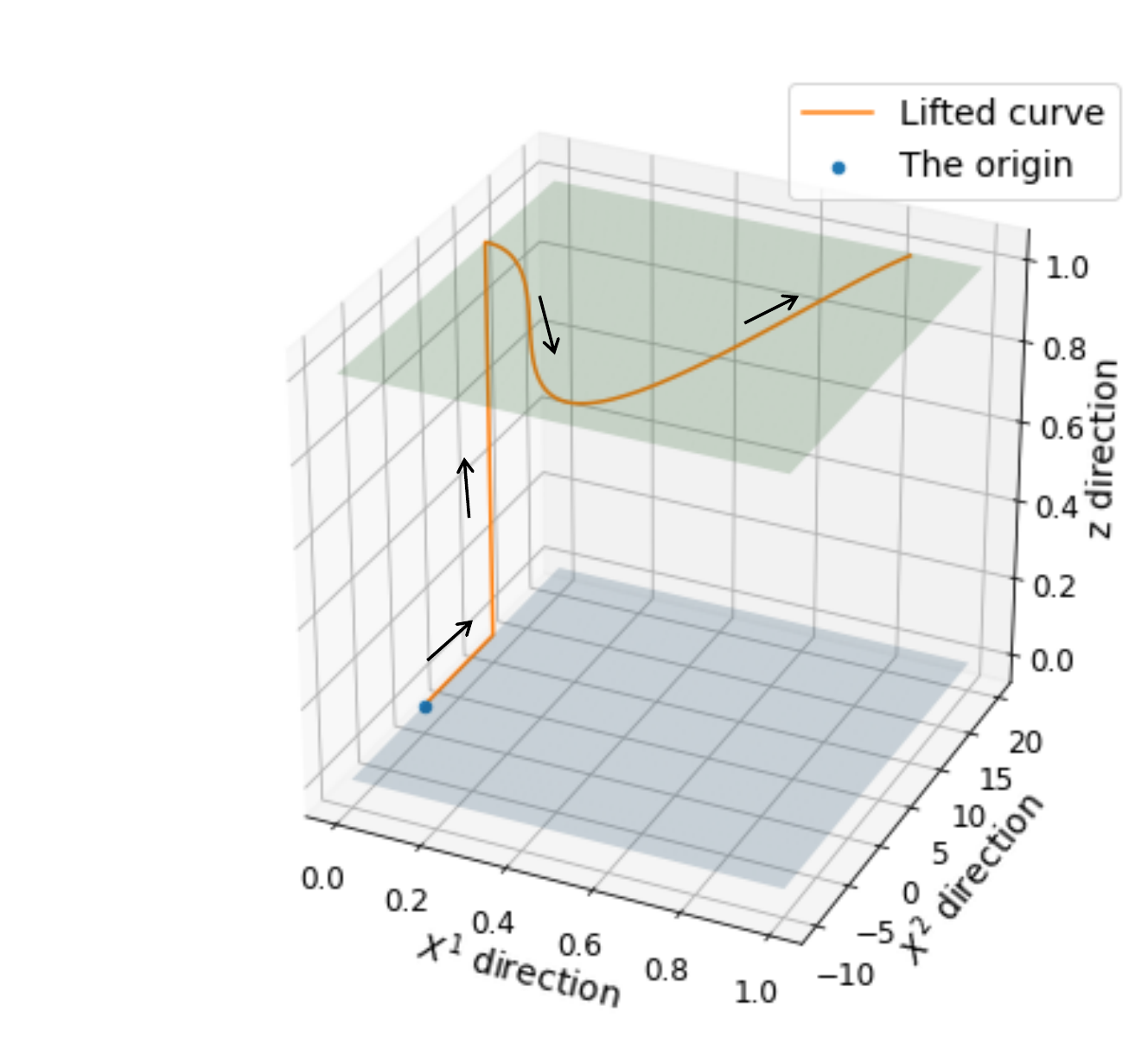}
\caption{The path after the $\mathbf{I}$-visibility transformation.}
\end{subfigure}
\begin{subfigure}[b]{0.48\linewidth}
\includegraphics[trim=3cm 14cm 4.5cm 3.5cm, clip, width=2.5in]{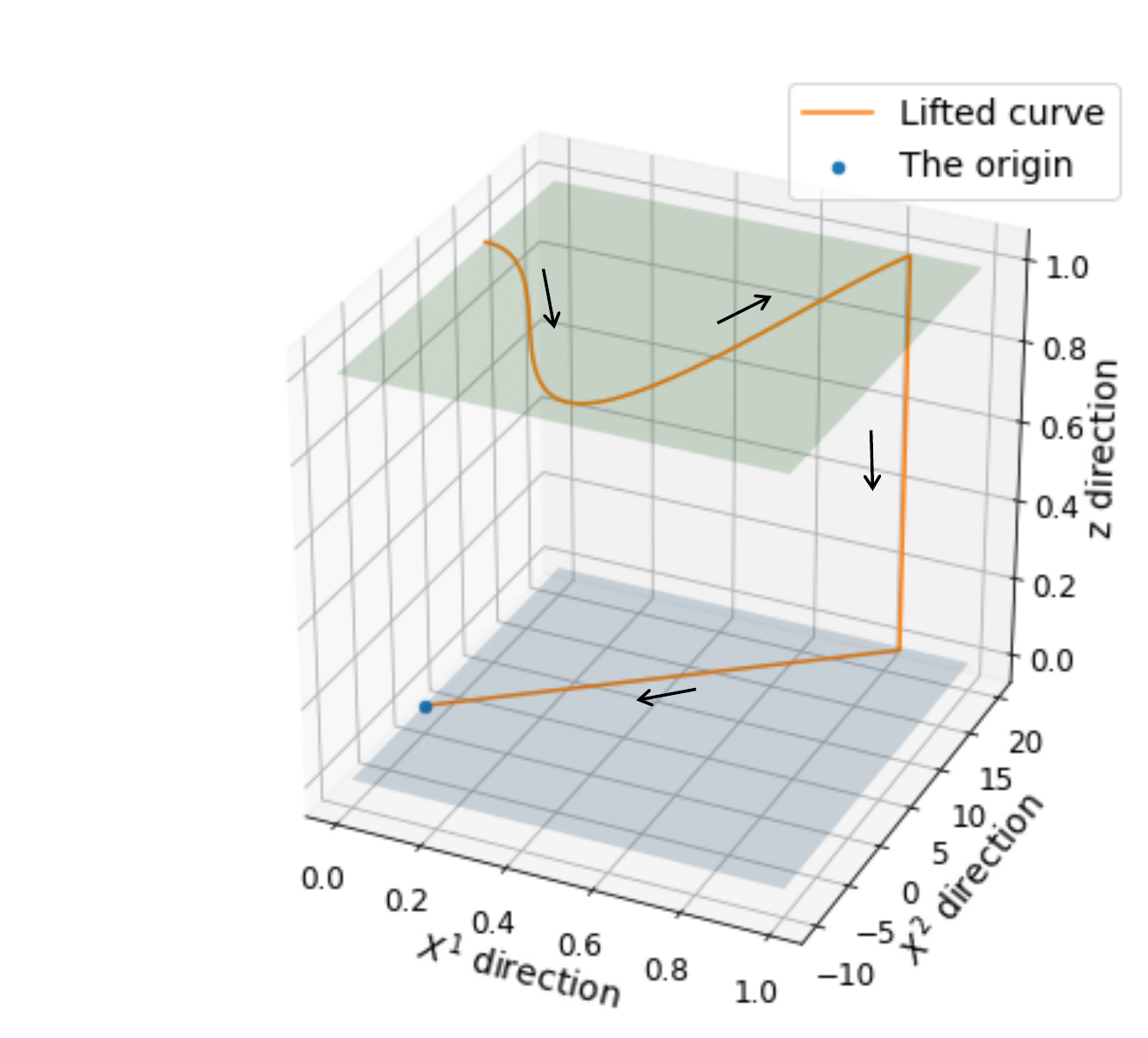}
\caption{The path after the $\mathbf{T}$-visibility transformation.}
\end{subfigure}
\caption{An illustration for visibility transformations on path $X$ from Figure \ref{fig:threelike}:
the left plot shows that the $\mathbf{I}$-visibility transformation transforms the 2-dimensional path to a 3-dimensional curve by joining the origin to the initial position of the original curve on the invisibility plane (z=0, the light blue plane) and lifting the 2-dimensional curve to the visibility plane (z=1, the light green plane); the right plot shows that the $\mathbf{T}$-visibility transformation by first lifting the 2-dimensional curve to the visibility plane and joining the tial positioon of the original curve to the origin.}
\label{fig:loop}
\end{figure} 

Figure \ref{fig:loop} illustrates that the structure of the path remains though being lifted to a higher dimensional space. Indeed, the visibility transformation preserves tree-like equivalence.
\begin{theorem}\label{thm: preserving}
Let $X$,$Y$ be continuous paths of finite length with $X \sim Y$. Then for the
${\mathbf{I}}$-visibility transformation we have 
$\gamma_{\mathbf{I}}(X)\sim \gamma_{\mathbf{I}}(Y)$ and $S(\gamma_{\mathbf{I}}(X))=S(\gamma_{\mathbf{I}}(Y))$. Similarly, for the ${\textbf{T}}$-visibility transformation, 
$\gamma_{\mathbf{T}}(X)\sim \gamma_{\mathbf{T}}(Y)$ and $S(\gamma_{\mathbf{T}}(X))=S(\gamma_{\mathbf{T}}(Y))$.
\end{theorem}
 In the following, we consider an ordered multi-index collection $J=(j_1,\ldots,j_{|J|})$ with $j_i \in [d+1]$ for $i\in [|J|]$, where $|\cdot|$ denotes the number of elements in a set. An application of Chen's identity leads to the following decomposition of the signature after the visibility transformation, showing that some terms in the signature of the path generated by the visibility transformation can be expressed in terms of the signature of the original path, i.e., $S(X)$, and the signature of the path $f_{\mathbf{I}(X)}$ (resp. $f_{\mathbf{T}(X)}$), i.e., $S(f_{\mathbf{I}(X)})$ (resp. $S(f_{\mathbf{T}(X)})$). 
   \begin{theorem}\label{thm:general2}
Let $X$ be a continuous path of finite length. For a multi-index collection $J$
{\small \begin{equation}\label{eqn:combination1}
 S(\gamma_{\mathbf{I}}(X))^J=\sum_{ \mathrel{\substack{(J_1| J_2)  =J\\ d+1\notin J_2} }} S_{f_{\mathbf{I}(X)}}^{J_1}S_{X}^{J_2}e_{m_1}\cdots e_{m_{|J_1|}} e_{h_1}\cdots e_{h_{|J_2|}},
\end{equation}}
and
{\small \begin{equation}\label{eqn:combination2}
 S(\gamma_{\mathbf{T}}(X))^J=\sum_{ \mathrel{\substack{(J_1| J_2)  =J\\ d+1\notin J_1} }} S_{X}^{J_1}S_{f_{\mathbf{T}(X)}}^{J_2}e_{m_1}\cdots e_{m_{|J_1|}} e_{h_1}\cdots e_{h_{|J_2|}}.
\end{equation}}
Here $J_1:=(m_1,\cdots,m_{|J_1|})$, $J_2:=(h_1,\cdots,h_{|J_2|})$ are multi-index collections, $(J_1| J_2)$ is a new multi-index collection in which $J_2$ is appended to $J_1$, and $S_X^{J_1}$,$S_{f_{\mathbf{I}(X)}}^{J_1}$ and $S_{f_{\mathbf{T}(X)}}^{J_2}$ are the corresponding coefficients of $S(X)$, $S(f_{\mathbf{I}(X)})$ and $S(f_{\mathbf{T}(X)})$.
\end{theorem}
Recall that $f_{\mathbf{I}(X)}$ (resp. $f_{\mathbf{T}(X)}$) is only related to the initial (resp. tail) position. These terms discussed in Theorem \ref{thm:general2} indeed capture the mixed effects of both the increments and the initial (resp. tail) position of the original path.  Theorem \ref{thm:general2} also leads to the following consequence.
\begin{corollary}\label{cor:simple}
Let $X$ be a continuous path of finite length. For $j\in [d]$, we have
{\small \begin{equation}
S(\gamma_{\mathbf{I}}(X))^j=X_1^je_j\ \ \text{and\ \ } S(\gamma_{\mathbf{T}}(X))^j=-X_0^je_j.
\end{equation}}
\end{corollary}
Clearly $X_1$ (resp. $X_0$) in the formula above is the tail (resp. initial) position value. Corollary \ref{cor:simple} shows the ${\mathbf{I}}$-visibility transformation (resp. the ${\mathbf{T}}$-visibility transformation) trivially captures the linear effect on the tail position (resp. the initial position) of the path. 

Based on the decomposition of the signature after the visibility transformation in Theorem \ref{thm:general2}, we can show that the signature of the lifted path after the ${\mathbf{I}}$-visibility transformation (resp. the ${\mathbf{T}}$-visibility transformation) preserves the effects of the increments of the original path and captures the (nonlinear) effects purely resulted from the initial (resp. tail) position.
\begin{theorem}\label{thm:general1}
Let there be given an $\mathbb{R}^d$-valued continuous path $X$ of finite length and a multi-index collection $J$ with $d+1\notin J$. Define ${J^-}:=(d+1|J)$, where $d+1$ is prefixed to $J$ on the left. Then
{\small \begin{equation}\label{eqn:old_new_equal1}
S(\gamma_{\mathbf{I}}(X))^{J^-}=S_X^J e_{d+1}e_{j_1}\ldots e_{j_{|J|}},
\end{equation}}
where $S_X^J$ is the corresponding coefficient of $S(X)$.
Similarly, define $J^{+}:=(J|d+1)$, where $d+1$ is postfixed to $J$ on the right. Then
{\small \begin{equation}\label{eqn:old_new_equal2}
 S(\gamma_{\mathbf{I}}(X))^{J^+}=\frac{1}{|J|!}\prod_{j\in J}X_0^je_{j_1}\cdots e_{j_{|J|}}e_{d+1}.
 \end{equation}}
Similarly, for the ${\mathbf{T}}$-visibility transformation, we have {\small 
\begin{equation}
S(\gamma_{\mathbf{T}}(X))^{J^+}=S_X^J e_{j_1}\ldots e_{j_{|J|}}e_{d+1}
 \end{equation}}
 and
{\small  \begin{equation}
S(\gamma_{\mathbf{T}}(X))^{J^-}=\frac{(-1)^{|J|+1}}{|J|!}\prod_{j\in J}X_1^je_{d+1}e_{j_1}\cdots e_{j_{|J|}}.
 \end{equation}}
\end{theorem}
For instance, Eqn. \eqref{eqn:old_new_equal1} directly shows the new feature covers the effect over segments of the original path, i.e., $S_X$.
\begin{rmk}
Another important object related to the signature is \emph{the log-signature}, which is the logarithm of the signature \cite{lyons2007differential}. The log-signature is a parsimonious description of the signature, while the (truncated) log-signature and signature are bijective. So the effect of the initial (resp. tail) position can be captured by log-signature with the visibility transformation too. In contrast to the signature, the log-signature offers the benefit for dimension reduction, but it should be combined with non-linear models for approximating any functional on the unparameterised path space \cite{liao2019learning}. 
\end{rmk}
\subsection{The path from streamed data}\label{sec:visibility2}
In the machine learning context, we often work on streamed data $\mathbf{x}=\big(\mathbf{x}_1,\ldots,\mathbf{x}_n\big)$, where $\mathbf{x}$ contains $n$ observations, and the $i$th observation $\mathbf{x}_i$, $i\in [n]$, is assumed to be a $d$-dimensional column vector at the $i$th time point. We assume the time for the $i$th observation is simply $i$. Later on, to extract the signature feature, the first step is to embed the time series data into a path over a continuous time interval. To do this, we usually construct a $\mathbb{R}^d$-valued continous path from $\mathbf{x}$ through \emph{piece-wise linear interpolation} along each coordinate dimension, denoted by $\mathbf{X}:[1,n]\to \mathbb{R}^d$. That is,
{\small \begin{equation}
\mathbf{X}_t^j=\mathbf{x}_{\lfloor t \rfloor}^j+(t-\lfloor t \rfloor)(\mathbf{x}_{\lfloor t \rfloor+1}^j-\mathbf{x}_{\lfloor t \rfloor}^j), \ \text{for\ }t\in[1,n]\ \text{and\ }j\in [d],
\end{equation}}
where $\lfloor \cdot \rfloor$ denotes the integer part of a real number. On the other hand, signature features or log-signature features can be obtained directly using discrete data through the well-established Python packages \emph{iisignature} \cite{reizenstein2018iisignature} and \emph{esig}, where the piecewise linear interpolation is implemented automatically by the packages. 

However, the easiest way to apply the ${\mathbf{I}}$-visibility transformation (resp. the ${\mathbf{T}}$-visibility transformation) on the generated path $\mathbf{X}$ is not directly following Definition \ref{def:vt}. Instead we may first expand the streamed data $\mathbf{x}$ from $n$ observations of $d$ dimensional vectors to $n+2$ observations of $d+1$ dimensional vectors, through the discrete ${\mathbf{I}}$-visibility transformation (resp. the discrete ${\mathbf{T}}$-visibility transformation) as follows:
\begin{definition}\label{def:discretevt}
Given a discrete data sequence $\mathbf{x}:=(\mathbf{x}_i)_{i\in [n]}$ where $\mathbf{x}_i$ are $d$-dimensional column vectors. The \emph{discrete ${\mathbf{I}}$-visibility transformation} maps $\mathbf{x}$ to $\bar{\mathbf{x}}:=(\bar{\mathbf{x}}_j)_{j\in [n+2]}$, where $\bar{\mathbf{x}}_j$ are $d+1$-dimensional column vectors for $j\in [n+2]$ and given by 
{\small \begin{equation}
\bar{\mathbf{x}}_1=\mathbf{0},\ \ \bar{\mathbf{x}}_2=[\mathbf{x}_1^1,\ldots, \mathbf{x}_1^d,0]^T,
\end{equation}}
and 
{\small \begin{equation}
\bar{\mathbf{x}}_{k+2}=[\mathbf{x}_{k}^1,\ldots, \mathbf{x}_{k}^d,1]^T,\text{\ for\ }k\in [n]. 
\end{equation}}
Here $\mathbf{0}$ is the $d+1$ dimensional zero vector and $A^T$ denotes the transpose of matrix $A$. 

Similarly, the \emph{discrete ${\mathbf{T}}$-visibility transformation} maps $\mathbf{x}$ to a new sequence $\hat{\mathbf{x}}:=(\hat{\mathbf{x}}_j)_{j\in [n+2]}$, where $\hat{\mathbf{x}}_j$ are $d+1$-dimensional column vectors and given by 
{\small \begin{equation}
\hat{\mathbf{x}}_{k}=[\mathbf{x}_{k}^1,\ldots, \mathbf{x}_{k}^d,1]^T,\text{\ for\ }k\in [n],
\end{equation}}
and
{\small \begin{equation}
 \hat{\mathbf{x}}_{n+1}=[\mathbf{x}_n^1,\ldots, \mathbf{x}_n^d,0]^T,\ \tilde{\mathbf{x}}_{n+2}=\mathbf{0}.
\end{equation}}
\end{definition}
Take a discrete data with two 2-dimensional observations $[1,2]^T, [3,4]^T$ for example, the new discrete data after the discrete ${\mathbf{T}}$-visibility transformation would be
{\small \begin{equation}
\begin{bmatrix}
0 \\
0 \\
0
\end{bmatrix}, 
\begin{bmatrix}
 1\\
 2 \\
 0
\end{bmatrix},
\begin{bmatrix}
 1\\
 2 \\
 1
\end{bmatrix},
\begin{bmatrix}
 3\\
 4 \\
 1
\end{bmatrix},
\end{equation}}
where the new data has four 3-dimensional observations. \footnote{To compute its signature or log-signature via python package, this new discrete data needs to be reshaped to a matrix of 4 rows and 3 columns.}

Then we generate a new path $\bar{\mathbf{X}}$ (resp.  $\hat{\mathbf{X}}$) through piece-wise linear interpolation on $\bar{\mathbf{x}}$ (resp. $\hat{\mathbf{x}}$). $\bar{\mathbf{X}}$ (resp.  $\hat{\mathbf{X}}$) is exactly the lifted path of  $\mathbf{X}$ after the $\mathbf{I}$-visibility transformation (resp. the $\mathbf{T}$-visibility transformation). This construction coincides with Definition \ref{def:vt}. In this sense, the discrete visibility transformation can be treated as an intermediate transformation. The availability of the aforementioned Python packages allows for signature feature extraction directly from data after the discrete visibility transformation.  Note that the \emph{invisibility-reset transformation} introduced in \cite{yang2017skeleton} is indeed the discrete ${\mathbf{T}}$-visibility transformation defined in Definition \ref{def:discretevt}.

Meanwhile, streamed data can be manipulated through other transformations together with the discrete visibility transformation. For example, the lead and lag transform \cite{gyurko2013leadlag1}, which accounts for the quadratic variability in data, was used in \cite{yang2017skeleton} along with the discrete ${\mathbf{T}}$-visibility transformation for human action recognition tasks. Note that all transformations should be done before applying the discrete visibility transformation, as illustrated in the pipeline (Figure \ref{fig:workflow}).
\begin{rmk}\label{rmk:timecost}
Theorem \ref{thm:general1} illustrates that the $p$th level signature of $\mathbf{X}$ is captured in the $(p+1)$th level signature of $\bar{\mathbf{X}}$. This implies that we may simply truncate the signature of $\bar{\mathbf{X}}$ to the $(p+1)$th level if the signature of $\mathbf{X}$ up to the $p$th level is needed. In this case, however, the number of signature terms computed increases from $\frac{d}{d-1}(d^p-1)$ to $\frac{(d+1)}{d}((d+1)^{p+1}-1)$, which leads to a growth in computational cost for extracting signature features. For example, for $d=2$ and $p=2$, the number of terms to be computed increases from $6$ to $39$. On the other hand, based on the assumption that position information provides extra features, embedding position information into the signature surely increases precision to some extent. Thus there is a trade-off between computational burden and accuracy. 

In the first numerical experiment (Section \ref{eg:handwriting}), to be fair and consistent, we compared signatures with/without visibility transformation truncated at the same level. In this case the number of signature terms computed are $\frac{d}{d-1}(d^p-1)$ to $\frac{d+1}{d}((d+1)^p-1)$ respectively. Alternatively, to further reduce the time cost, one may compute log-signature features of $\bar{\mathbf{X}}$ instead.
\end{rmk}
\section{Applications}\label{sec:app}
\subsection{Classification of handwriting data}\label{eg:handwriting}
The concept of the visibility transformation originated in the analysis of temporal handwriting data, where the path in the visibility space indicates when one is writing the letter and the pen is visible. Apparently where one starts his/her writing is also crucial for making classification decisions. Here we chose character trajectories data set \cite{duo2019UCI} for assessment, where multiple, labelled samples of pen tip trajectories were recorded whilst writing individual characters \cite{williams2006handwriting, williams2007handwriting, williams2008handwriting}. The data consists of 2858 instances for 20 different characters, and was captured using a WACOM tablet at 200Hz. Each character sample is a 3-dimensional pen tip velocity trajectory, namely $(x,y,p)$, where $(x,y)$ is the trajectory and the $p$ coordinate represents the pen tip force. The lengths of the samples for the same character are not necessarily the same. 

The original handwriting data contains training set (50\%) and testing set (50\%). To account for quadratic variability of the path, we combined the $\mathbf{I}$-visibility transformation and the lead lag transform as described at the end of Section \ref{sec:visibility2}. All the data were then transformed to three different features respectively: truncated signatures with the lead lag transform (LLT), truncated signatures with LLT and being prefixed by the explicit initial position (IP), and truncated signatures with LLT and the $\mathbf{I}$-visibility transformation (IVT). We also extracted truncated signature features with LLT and IVT on the trajectory, namely the $(x,y)$ path only, where we ignored the pen tip force. For each feature method and the corresponding transformed training set, a lightGBM model was trained for classification with hyperparameter tuning implemented via grid search with cross validation. The signature-based lightGBM models for classifying 20 different characters were tested repeatedly so that the performance and the randomness of the models can be captured in terms of the average accuracy and standard deviation in Figure \ref{fig:accuracy1}. In the experiment, the signature features were truncated to levels $\{2,3,4,5,6\}$. We also briefly compared the computational efficiency of the four features in Figure \ref{fig:accuracy2}.

The increasing trend of the accuracy curves for the four models with $n$ suggests higher order terms of the signature are not redundant in this case. Across the three models on the full path $(x,y,p)$, the performance of the model with the visibility transformation is the best and the one with signature features alone is the worst across all $n$, which illustrates the power of the visibility transformation method in such applications. It is also worth noticing that the performance of classification using signature features with the visibility transformation on partial path $(x,y)$ is also superior to the classification using either signature features alone or prefixed initial position features on the full path $(x,y,p)$. This may imply that either the pen tip force dimension may not be too informative or our proposed method is very efficient in seizing pivotal and non-redundant information from limited data.

\begin{figure}[h!]
  \centering
\begin{subfigure}[b]{0.48\linewidth}
\includegraphics[width=2.5in]{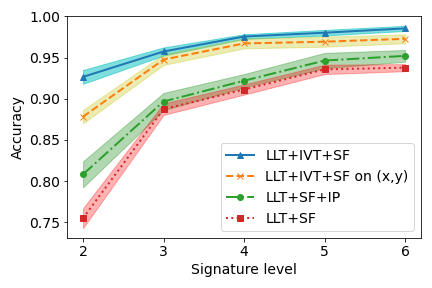}
\caption{Signature level VS average accuracy with standard deviation.\label{fig:accuracy1}}
\end{subfigure}
\begin{subfigure}[b]{0.48\linewidth}
\includegraphics[width=2.5in]{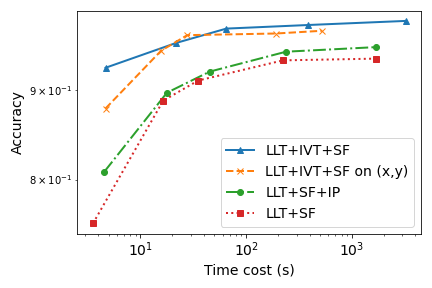}
\caption{CPU time VS average accuracy (log-log scale).\label{fig:accuracy2}}
\end{subfigure}
\caption{Average accuracy curves for handwriting classification with different signature features, where 'LLT' is short for the lead lag transform , 'IVT' short for the $\mathbf{I}$-visibility transformation, 'SF' short for signature feature and "IP" short for the initial position.}
\label{fig:accuracy}
\end{figure}

Clearly, the lightGBM models based on truncated signatures with IVT are more efficient than the models without in terms of computational cost as shown in Figure \ref{fig:accuracy2}. From Remark \ref{rmk:timecost}, the time cost of the model based on truncated signatures with IVT and LLT is approximately 1.3 to 2.5-times more expensive than the time cost of the model based on truncated signatures with LLT and IP at the same signature level. But due to its better accuracy the former model is superior for all time costs. Besides, compared to the model based on truncated signature features (the LLT+SF curve) at each fixed time cost, the higher accuracy from the model based on truncated signatures with the explicit initial position (the LLT+SF+IP curve)  indicates that even the linear effect of the initial position provides additional information for this learning task.
\begin{table}[!t]
\renewcommand{\arraystretch}{1.2}
\caption{Comparison for handwriting classification with different methods.}
\label{tab:3}
\centering
\begin{tabular}{|c||c|}
\hline 
\bf Method  & \bf Accuracy  \\
\hline 
\textbf{VHEM}- \textbf{H3M} \cite{coviello2014clustering}  &  65.10\%  \\
\hline
\textbf{FK} \cite{jaakkola1999generative}  & 89.26\%  \\
\hline
$\phi$(\textbf{O},\textbf{HMM})+\textbf{SVM} \cite{perina2009entropydistance} & 92.91\%  \\
\hline
\textbf{TK} \cite{tsuda2002kernel}  & 93.67\%  \\
\hline 
\textbf{LLT}+\textbf{IVT}+\textbf{SF} on $(x,y)$ & \bf 97.27\%  \\
\hline  
\textbf{SDD} \cite{grabocka2016shapelet} &  98.00\% \\
\hline  
\textbf{MCDS} \cite{iosifidis2012fuzzy} &  98.25\% \\
\hline 
\textbf{LLT}+\textbf{IVT}+\textbf{SF} & \bf 98.54\%  \\
\hline 
\end{tabular}
\end{table}

In Table \ref{tab:3}, the best performance of our method is compared with that of
other methods proposed in the literature on this handwritten character database. Row $2$ to $5$ are results from classification based on different hidden Markov models (HMMs): the cluster HMM on the hierarchical expectation–maximization (\textbf{VHEM}- \textbf{H3M}) \cite{coviello2014clustering}, the support vector machine on features from HMM embedded entropy feature extractor ($\phi$(\textbf{O},\textbf{HMM})+\textbf{SVM}, \cite{perina2009entropydistance}), and the generative classification based on HMM and the Fisher (\textbf{FK}) and TOP Kernels (\textbf{TK}). The best performance of our method (\textbf{LLT}+\textbf{IVT}+\textbf{SF}) can be as high as 98.54\%, and the best performance of our method on partial path $(x,y)$ can achieve 97.27\%. Even the worst performance of \textbf{LLT}+\textbf{IVT}+\textbf{SF} (at level 2) is beyond all the HMM-related classifications. Comparing to the scalable shapelet discovery (\textbf{SDD}) in \cite{grabocka2016shapelet} and the modified clustering discriminant analysis (\textbf{MCDS}) \cite{iosifidis2012fuzzy} based on an iterative optimisation procedure which both provider high precision, the implementation of our proposed method is much easier with well-designed Python packages as mentioned before.
\subsection{Gesture classification: Chalearn 2013 data}\label{eg:Chalearn}
This example is borrowed from \cite{liao2019learning}.  The ChaLearn 2013 multi-modal gesture datase  \cite{escalera2013ChaLearn} contains 23 hours of Kinect data of 27 subjects performing 20 Italian gestures. The data includes RGB, depth, foreground segmentations and full body skeletons. In \cite{liao2019learning}, a log-signature-based recurrent neural network model (PT-Logsig-RNN, see Figure 3 in \cite{liao2019learning}), which consists with path transformation layers, a log-signature (sequence) layer,  an initial position layer\footnote{The initial position layer was not visualised in Figure 3 \cite{liao2019learning} but is parallel to the log-signature layer.}, the RNN-type layer and the last fully connected layer, was built for gesture recognition on skeleton data and achieved state-of-the-art accuracy. Put in other words, the architecture utilised the log-signature of paths together with the explicit initial position values for gesture classification tasks. 

Following \cite{liao2019learning}, we use only skeleton data for the gesture recognition. To classify the 20 gestures using log-signature features with the visibility transformation, we modified the architecture of PT-Logsig-RNN Model by adding a visibility transformation layer to the end of path transformation layers, which is consistent with the proposed workflow of the pipeline (Figure \ref{fig:workflow}), and deleting the initial position layer. We attempted both the ${\mathbf{T}}$-visibility transformation and ${\mathbf{I}}$-visibility transformation. In contrast with the experience in \cite{yang2017skeleton}, the best performance for this task achieved when truncating the log-signature at level $3$ with the ${\mathbf{I}}$-visibility transformation. This implies that applying either ${\mathbf{I}}$-visibility transformation or ${\mathbf{T}}$-visibility transformation is dataset-dependent.
\begin{table}[!t]
\renewcommand{\arraystretch}{1.2}
\caption{Comparison of different methods on the Chalearn 2013 data.}
\label{tab:1}
\centering
\begin{tabular}{|c||c|}
\hline 
\bf Method  & \bf Accuracy  \\
\hline 
{\bf Deep LSTM} \cite{shahroudy20163dhumanactivity} &  87.10\%  \\
\hline
{\bf Two-stream LSTM} \cite{wang2017twostream} & 91.70\%  \\
\hline
{\bf ST-LSTM + Trust Gate} \cite{liu2017trustgates}  & 92.00\%  \\
\hline
{\bf Three-stream net TTM}\cite{li2019severalfullconnected} & 92.08\%  \\
\hline 
{\bf PT-Logsig-RNN} \cite{liao2019learning} & 92.21\%  \\
\hline  
{\bf Modified PT-Logsig-RNN}&  \bf 92.89\% \\
\hline 
\end{tabular}
\end{table}
We compared this modified PT-Logsig-RNN model (level 3) with several state-of-the-art methods in Table \ref{tab:1}. Row 2 to Row 5 are results from LSTM-based models: a recurrent neural network structure to model the long-term temporal correlation of the features for each body part ({\bf Deep LSTM})\cite{shahroudy20163dhumanactivity}, a two-stream RNN architecture to model both temporal dynamics and spatial configurations for skeleton based action
recognition ({\bf Two-stream LSTM})\cite{wang2017twostream}, a spatio-temporal LSTM with an additional gate to analyze the
reliability of the input measurements ({\bf ST-LSTM + Trust Gate})\cite{liu2017trustgates}, and a three-stream temporal transformer module to actively transform the input data temporally and finally adjust the key frame to the best time stamp for the network ({\bf Three-stream net TTM})\cite{li2019severalfullconnected}. Our modified PT-Logsig-RNN model (level 3) achieved the state-of-the-art accuracy.

Compared to the original PT-Logsig-RNN model, the higher accuracy of the modified PT-Logsig-RNN model (level 3) is a result of applying visibility transformation. It is also worth noticing that the visibility transformation allows easy implementation/modification without changing the original pipeline of neural networks as illustrated in the workflow in Figure \ref{fig:workflow}.

\section{Conclusion}\label{sec:con}
To capture the informative information from streamed data for learning tasks, we explore a transformation that encodes the effects on the absolute position of streamed data into signature features with theoretical justifications. The enhanced feature is unified, theoretical-backed, and simple to implement with. It is superior to many benchmark methods that require handy data preparation and implementation of complicated algorithms in applications when absolute position of the data is intrusive. 

\section*{Acknowledgements}
YW, HN, TJL were supported by the Alan Turing Institute under the EPSRC
grant EP/N510129/1 and by EPSRC under EP/S026347/1.

\section*{Appendix}

\begin{proof}[Proof of Theorem \ref{thm: preserving}]
For two continuous bounded variation paths $X$ and $Y$, the paths generated by the $\mathbf{I}$-visibility transformation can be denoted as $f_X*L_X$ and $f_Y*L_Y$ according to Definition \ref{def:vt}. The assumption $X\sim Y$ leads to $L_X\sim L_Y$ naturally. Meanwhile, $X\sim Y$ implies $\mathbf{I}(X)=\mathbf{I}(Y)$ and $\mathbf{T}(X)=\mathbf{T}(Y)$. This further implies that $f_X=f_Y$. Then we can conclude that $f_X*L_X\sim f_Y*L_Y$ by using the fact that the concatenation respects $\sim$ \cite{hambly2010uniqueness}. 

The assertion for the $\mathbf{T}$-visibility transformation follows a similar argument.

\end{proof}

\begin{proof}[Proof of Theorem \ref{thm:general2}]
For a multi-index collection $J$ such that $d+1\notin J$, it is very clear that $S(L_X)^J=S(X)^J$. Note that for any multi-index collection $I$ such that $d+1\in I$, $S(L_X)^I=0$. Then the assertion follows from Chen's identity (Theorem \ref{thm:chen}) and tensor product of \eqref{eqn:power_rep}.
\end{proof}
\begin{proof}[Proof of Theorem \ref{thm:general1}]
Eqn.\eqref{eqn:combination1} in Theorem \ref{thm:general2} leads to
{\small \begin{align*}
&S(\gamma_{\mathbf{I}}(X))^{J^-}=(S(f_{\mathbf{I}(X)})\otimes S(L_X))^{J^-}\\
    &=\sum_{ \mathrel{\substack{(J_1|J_2) =J^-\\ d+1\notin J_2} }} S_{f_{\mathbf{I}(X)}}^{J_1}S_{X}^{J_2}e_{m_1}\cdots e_{m_{|J_1|}} e_{h_1}\cdots e_{h_{|J_2|}},
\end{align*}}
where $J_1:=(m_1,\ldots,m_{|J_1|})$, $J_2:=(h_1,\ldots,h_{|J_2|})$, $S_{f_{\mathbf{I}(X)}}$ and $S_{X}$ are the corresponding coefficients of $S(f_{\mathbf{I}(X)})$ and $S(X)$ respectively.
On the other hand, for the collection $J_1$ with $|J_1|\geq 2$, i.e., $J_1=(d+1,h_2,\cdots,h_{|J_2|})$, $S_{f_{\mathbf{I}(X)}}^{J_1}=0$. This can be shown by induction, and is omitted here. Thus the only non-vanishing term is when $J_1=(d+1)$, where $S_{f_{\mathbf{I}(X)}}^{J_2}=1$.
In total, we have that by setting $J_2=J$
{\small \begin{align*}
&S(\gamma_{\mathbf{I}}(X))^{J^-}= S_{f_{\mathbf{I}(X)}}^{d+1}S_{X}^{J}e_{d+1}e_{j_1}\cdots e_{j_{|J|}}=S_{X}^{J}e_{d+1}e_{j_1}\cdots e_{j_{|J|}}.
\end{align*}}
For the second part, with multi-index ${J^+}:=(J|d+1)$, using a similar argument as for $J^-$, we can conclude from Eqn.\eqref{eqn:combination1} in Theorem \ref{thm:general2} again that
{\small \begin{align*}
&S(\gamma_{\mathbf{I}}(X))^{J^+}=(S(f_{\mathbf{I}(X)})\otimes S(L_X))^{J^+}
    =S_{f_{\mathbf{I}(X)}}^{J^+}e_{j_1},\ldots,e_{j_{|J|}}e_{d+1}.
\end{align*}}
Now it remains to compute $S_{f_{\mathbf{I}(X)}}^{J^+}$. For non-empty $J$, by induction again we can show that 
{\small\begin{align*}S_{f_{\mathbf{I}(X)}}^{J^+}=\frac{1}{|J|!}\prod_{j\in J}X_0^j.
\end{align*}}
The assertion for the $\mathbf{T}$-visibility transformation follows a similar argument.
\end{proof}

\end{document}